\documentclass[letterpaper]{article} % DO NOT CHANGE THIS
\usepackage[draft]{aaai2026}  % DO NOT CHANGE THIS
\usepackage{times}  % DO NOT CHANGE THIS
\usepackage{helvet}  % DO NOT CHANGE THIS
\usepackage{courier}  % DO NOT CHANGE THIS
\usepackage[hyphens]{url}  % DO NOT CHANGE THIS
\usepackage{graphicx} % DO NOT CHANGE THIS
\usepackage{natbib}  % DO NOT CHANGE THIS AND DO NOT ADD ANY OPTIONS TO IT
\usepackage{caption} % DO NOT CHANGE THIS AND DO NOT ADD ANY OPTIONS TO IT
\usepackage{algorithm}
\usepackage{algorithmic}

\usepackage{newfloat}
\usepackage{listings}
\DeclareCaptionStyle{ruled}{labelfont=normalfont,labelsep=colon,strut=off} % DO NOT CHANGE THIS
\floatstyle{ruled}
\newfloat{listing}{tb}{lst}{}
\floatname{listing}{Listing}
\usepackage{amsmath,amsthm}

\title{Learning Vision-Based Neural Network Controllers with Semi-Probabilistic Safety Guarantees}
\author {
    Xinhang Ma\textsuperscript{\rm 1},
    Junlin Wu\textsuperscript{\rm 1},
    Hussein Sibai\textsuperscript{\rm 1},
    Yiannis Kantaros\textsuperscript{\rm 1},
    Yevgeniy Vorobeychik\textsuperscript{\rm 1}
}
\affiliations {
    \textsuperscript{\rm 1}Washington University in St. Louis\\
    \{m.owen, junlin.uw, ioannisk, sibai, yvorobeychik\}@wustl.edu
}

\newtheorem{assumption}{Assumption}
\newtheorem{theorem}{Theorem}

\begin{document}

\maketitle

\begin{abstract}
Ensuring safety in autonomous systems with vision-based control remains a critical challenge due to the high dimensionality of image inputs and the fact that the relationship between true system state and its visual manifestation is unknown. 
Existing methods for learning-based control in such settings typically lack formal safety guarantees.
%often lack formal safety guarantees, particularly when neural network controllers operate on images. 
To address this challenge, we introduce a novel semi-probabilistic verification framework that integrates reachability analysis with conditional generative networks and distribution-free tail bounds to enable efficient and scalable verification of vision-based neural network controllers.
Next, we develop a gradient-based training approach that employs a novel safety loss function, safety-aware data-sampling strategy to efficiently select and store critical training examples, and curriculum learning, to efficiently synthesize safe controllers in the semi-probabilistic framework.
%Our approach employs curriculum learning and a novel safety-aware data-sampling strategy to efficiently select and store critical training examples, significantly reducing training complexity. 
%We also provide rigorous theoretical guarantees bounding the proportion of the state space verified as safe. 
Empirical evaluations in X-Plane 11 airplane landing simulation, CARLA-simulated autonomous lane following, F1Tenth vehicle lane following in a physical visually-rich miniature environment, and Airsim-simulated drone navigation and obstacle avoidance demonstrate the effectiveness of our method in achieving formal safety guarantees while maintaining strong nominal performance.
\end{abstract}

\begin{links}
    \link{Code}{https://github.com/xhOwenMa/SPVT}
    % \link{Datasets}{https://aaai.org/example/datasets}
    % \link{Extended version}{https://aaai.org/example/extended-version}
\end{links}

% the guidelines from AAAI are moved away:
% \input{assets/guidelines}

% paper sections:
\section{Introduction}\label{sec:intro}

Many real-world applications, such as self-driving cars and robotic navigation, require controllers that process high-dimensional image inputs to make real-time decisions. 
The centrality of visual inputs (particularly when other modalities are limited or unreliable) thus makes ensuring the safety of vision-based control an important problem in trustworthy AI. 
However, verifying the safety of such controllers remains a major open challenge due to the complexity of image-based inputs and the high computational cost of traditional verification methods~\cite{huang2019reachnn, katz2017reluplex, huang2017safety}.

While reinforcement learning (RL) with high-dimensional image inputs has shown promise in learning control policies that optimize performance~\cite{le2022deep, kiran2021deep}, most methods 
%(even under the constrained reinforcement learning framework~\cite{achiam2017constrained}) 
lack \textit{formal guarantees} of safety~\cite{kochdumper2023provably}. Moreover, verifying the safety of neural network controllers operating in high-dimensional observation spaces remains computationally intractable. 
Existing approaches to safe control primarily focus on low-dimensional state inputs~\cite{dawson2022safe} and empirical safety evaluations. Other approaches, such as verified safety over the entire input region and control barrier function-based methods, have also been explored~\cite{wu2024verified, emam2022safe}. However, these methods struggle when the controller operates on image inputs due to the high dimensionality of the observation space.
Moreover, while dynamic behavior of many autonomous systems of interest has established models approximating their trajectories through the system state space, physics that map state to its visual representation are considerably more complex, and associated models far more complex and less reliable. 
%\junlin{reduce one word to save space}

In this work, we integrate reachability analysis with generative modeling to enable efficient verification of neural network controllers operating on high-dimensional image spaces~\cite{katz2022verification, cai2024scalable}. Specifically, we employ a generative neural network with latent inputs representing environmental and perceptual variation to model the perceptual mapping from states to images, allowing us to verify safety properties in a structured and lower-dimensional latent space. To address the scalability challenges of verification, we introduce a \textit{semi-probabilistic verification (SPV) framework}, where safety properties are verified over a sampled distribution of initial states (using distribution-free tail bounds) but for all possible latent environment representations of the trained generative perception model. 
In addition, we present a training algorithm that makes use of a novel safety loss as a differentiable proxy to this verification objective.
A key component of this algorithm is our approach to adapt the training set, stochastically biasing it towards states for which safety is difficult to verify.
Our experiments in simulated plane landing,  simulated and physical autonomous lane following, and drone navigation and obstacle avoidance, demonstrate that the proposed approach yields policies that exhibit significantly stronger safety properties compared to state-of-the-art safe control baselines.
%propose a training framework that incorporates this verification objective directly into training to gradually increase the verified safety horizon while ensuring high control performance.

In summary, our key contributions are as follows:
\begin{itemize}
    \item A novel semi-probabilistic safety verification (SPV) framework that provides formal safety guarantees while remaining computationally feasible in high-dimensional vision-based control settings.
    \item A novel training approach which uses a differentiable proxy loss for SPV and maintains a dynamic training set which adaptively prioritizes safety-critical states.
    %in training. 
    %that stores critical training examples while ensures controllers' performances in typical cases.
    \item Experimental evaluation in simulated and physical path following and simulated drone control that demonstrates effective empirical and provably-verified performance of the policies trained through our approach in comparison with  state-of-the-art baselines.
    %We establish theoretical safety guarantees that bound the proportion of the state space formally verified as safe.
\end{itemize}

\section{Model}

\subsection{Problem Formulation}\label{sec:methods:formulation}
We consider a discrete-time dynamical system:
% \begin{align}
% \label{E:dynsys}
% \begin{split}
%     &s_{t+1} = f(s_t, u_t)\\
%     &o_t = h(s_t,\omega)\\
%     &s_0 \sim \mathcal{D}, \ \ \omega \sim \Omega,
% \end{split}
% \end{align}
\begin{align}
\label{E:dynsys}
    s_{t+1} = f(s_t, u_t),
    o_t = h(s_t,\omega),
    s_0 \sim \mathcal{D}, \ \ \omega \sim \Omega,
\end{align}
where $s_t \in \mathcal{S}$ is the system state (e.g., position, steering angle of the vehicle), $o_t \in \mathcal{O}$ is the vision-based (image) observation perceived by the agent, $u_t \in \mathcal{U}$ the control action, $h$ the mapping from state to observation, and $\mathcal{D}$ a distribution over the initial state $s_0$.
Notably, $h$ takes as input a \emph{perceptual environment} $\omega$, which models an unobserved source of environment variation distributed according to an unknown distribution $\Omega$.
We assume that the dynamics $f$ are known (for example, well-known dynamical system models for physical systems), while $h$ and $\mathcal{D}$ are both unknown.
At execution time, we suppose that only observations $o_t$ are known to the controller, with state $s_t$ unobservable.
Our goal is to synthesize a control policy $\pi$ mapping visual observations $o$ to control actions $u$ which is \emph{provably safe} in the sense we formalize next.

Let $P$ denote a safety specification, which is a predicate $P(s)$ indicating whether a state $s \in \mathcal{S}$ is safe or not.
Similarly, let $P(\mathcal{T})$ indicate whether $P(s)$ is true for all $s \in \mathcal{T} \subseteq \mathcal{S}$.
We assume that both $P(s)$ and $P(\mathcal{T})$ can be evaluated efficiently (for example, safety is often described using linear inequalities and $\mathcal{T}$ is a polyhedron).
Given a policy $\pi$, the controlled dynamical system effectively becomes
\(s_{t+1} = f(s_t,\pi(h(s_t, \omega))\).
We say that this dynamical system is \emph{safe} for an initial state $s_0$ over a horizon $K$ if  $P(s_t)$ is true for all $0 \le t \le K$.

This notion of safety, however, is limited for two reasons.
First, we do not know $h$, so we cannot directly verify the dynamical system above.
Second, we wish for a policy $\pi$ to satisfy safety in a way that is not tied to a specific starting state, but with respect to the full set of initial states $\mathcal{S}$.
%but is rather a property of $\pi$ within the dynamical system model~\eqref{E:dynsys}.
%\footnote{\textcolor{red}{Do you make any assumptions on the initial state? Eg does it belong to a fixed set? It is a bit unclear what this means `but is rather a property of $\pi$ within the dynamical system model'}}
We address the first challenge by leveraging a conditional generative neural network model to approximate $h$, and the second by using a \emph{semi-probabilistic verification (SPV)} framework.
We describe both of these ideas next.

%Given a neural network verifier $V$ and a safety specification $P$, our objective is to train a vision-based neural network controller $\pi: \mathcal{I} \rightarrow \mathcal{A}$ that can be verified as safe using $V$. The controller maps from an image space $\mathcal{I}$ to an action space $\mathcal{A}$, ensuring that the trained controller satisfies the safety properties $P$ under all plausible image inputs.

\subsection{Approximating the Visual Observation Model}

We address the first challenge by using a conditional generative model $g(s,z)$ which induces a distribution over observations $o \in \mathcal{O}$ for a given state $s \in \mathcal{S}$, with $z \in \mathcal{Z}$ a (typically distributed according to a uniform or Gaussian distribution) random vector, analogous to the approach proposed by~\citet{katz2022verification}.
Such a generator can be trained, for example, using the conditional generative adversarial network (cGAN) framework~\citep{mirza2014conditional,isola2017image} or a conditional diffusion model~\citep{yang2023lossy} from a collection of data $(o,s)$ in which images $o$ are annotated with associated states $s$. 
We can view the latent random vectors $z$ as representations of natural environment variation (e.g., different perspectives, lighting, etc).
The goal here is that $g$ approximates $h$, but in practice this assumption is too strong.
Instead, we make the following considerably weaker assumption about the relationship between $h$ and $g$.
\begin{assumption}
\label{Assumption:hg}
    $\sup_{s,\omega} \inf_z \|h(s,\omega)-g(s,z)\| \le \epsilon$.
\end{assumption}
In practice, this assumption boils down to having (a) sufficient training data for the generator $g$ and (b) a sufficiently rich representation (e.g., neural network) and latent dimension of $z$.
We validate this assumption in Section~\ref{sec:experiment} below.

\section{Semi-Probabilistic Verification}
\label{sec:methods:verification}
Our notion of safety is based on
$K$-reachability.
In traditional $K$-reachability, safety is guaranteed for all states in some specified set $\mathcal{S}_0 \subseteq \mathcal{S}$ from which dynamics may be initialized.
However, such $K$-reachability proofs are generally conservative and typically suffer from significant scalability challenges.
%In domains 
%such as ours 
When controllers use vision, 
scalability can be a prohibitive barrier to verification.
In practice, however, we can often obtain information about the distribution $\mathcal{D}$ over initial states $s_0$ of the dynamical system. For example, by collecting empirical visual data and annotating it with state-relevant information (Waymo~\cite{sun2020scalability} or KITTI~\cite{geiger2013vision} datasets in the case of autonomous driving).
On the other hand, the distribution over the initial state is often difficult to cleanly characterize (indeed, it may be heavy-tailed).
It is, therefore, natural to appeal to distribution-free bounds to obtain probabilistic safety proofs with respect to the unknown distribution $\mathcal{D}$ over $s_0$ based on safety properties obtained for a finite sample of initial states.
In contrast, the distribution of the visual environment induced by $\omega$ is far more challenging to characterize or sample, particularly since we do not know $h$.

We propose to balance these considerations through a semi-probabilistic verification (SPV) framework, in which we aim to obtain provable distribution-free guarantees with respect to $\mathcal{D}$, but which hold in the worst case with respect to environment variation $\omega$.

To formalize, fix a policy $\pi$ and let
\begin{align*}
\mathcal{S}_{t+1}(s_0,\pi) = \{f(s,\pi(o))|&o=h(s,\omega),\\
&s\in \mathcal{S}_t(s_0,\pi),\omega \in \Omega\},
\end{align*}
where $\mathcal{S}_0(s_0,\pi) = \{s_0\}$.
Define $\mathrm{Reach}_K(s_0,\pi) = \cup_{t =0}^K \mathcal{S}_t(s_0,\pi)$, that is, all states that can be reached from $s_0$ for any perceptual environment $\omega \in \Omega$.
Note that this form of reachability cannot be verified, since we do not know $h$.
However, we can now leverage the generator $g$ as a proxy, with Assumption~\ref{Assumption:hg} allowing us to obtain sound safety guarantees.
Specifically, let 
\begin{align*}
\hat{\mathcal{S}}_{t+1}(s_0,\pi) = \{f(s,\pi(o))|&\|o - g(s,z)\| \le \epsilon,\\
&s\in \hat{\mathcal{S}}_t(s_0,\pi),z \in \mathcal{Z}\},
\end{align*}
and define $\mathrm{Reach}_K(s_0,\pi,g) = \cup_{t =0}^K \hat{\mathcal{S}}_t(s_0,\pi)$.
The following result allows us to focus on verification with respect to $\mathrm{Reach}_K(s_0,\pi,g)$.
\begin{theorem}
\label{T:approxvis}
    Under Assumption~\eqref{Assumption:hg}, $\mathrm{Reach}_K(s_0,\pi)\subseteq \mathrm{Reach}_K(s_0,\pi,g)$.
    Therefore, $P(\mathrm{Reach}_K(s_0,\pi,g)) \Rightarrow P(\mathrm{Reach}_K(s_0,\pi))$.
\end{theorem}
\begin{proof}
        Suppose $\mathcal{S}_t \subseteq \hat{\mathcal{S}}_t$ and let $s \in \mathcal{S}_t$ and $o = h(s,\omega)$ for some $\omega \in \Omega$.
    Then $s \in \hat{\mathcal{S}}_t(s_0,\pi)$
    and by Assumption~\eqref{Assumption:hg}, $\|o - g(s,z)\| \le \epsilon$ for some $z \in \mathcal{Z}$.
    Consequently, $f(s,\pi(o)) \in \hat{\mathcal{S}}_{t+1}(s_0,\pi)$.
    Since $\hat{\mathcal{S}}_0(s_0,\pi) = \mathcal{S}_0(s_0,\pi) = \{s_0\}$, the result follows by induction.
\end{proof}
In practice, we will make use of a verification tool that is able to efficiently obtain an over-approximation of $\mathrm{Reach}_K(s_0,\pi,g)$, which maintains soundness.

Our next step is to combine this with a distribution-free tail bound with respect to the initial state distribution $\mathcal{D}$.
Specifically, suppose that we have a finite sample of $N$ initial states $\{s_i\}_{i=1}^N$ i.i.d.~from $\mathcal{D}$.
%\subsection{Probabilistic Safety Guarantee}
%In this section, we provide a theoretical proof 
Next we show that by verifying only with respect to this finite sample of $N$ states, we can achieve a semi-probabilistic safety guarantee for the entire initial region with respect to the unknown distribution $\mathcal{D}$.
\begin{theorem}\label{thm:hoeffding_bound}
%Let $\mathcal{X}\subseteq\mathcal{S}$ be a region in the state space.
%, and let $P(s)$ be a binary safety predicate that indicates whether a state $s \in \mathcal{X}$ satisfies a given safety property. 
%Suppose we randomly sample $N$ points $\{s_1, s_2, \dots, s_N\}$ from $\mathcal{X}$ according to a uniform distribution and find that $m$ of these points satisfy $P(s)$.
Suppose that $\{s_i\}_{i=1}^N$ i.i.d.~from $\mathcal{D}$ and let $V = \{s_i|P(\mathrm{Reach}_K(s_i,\pi,g))\}$ be the subset of sampled initial states for which the reachable set is safe. %\footnote{\textcolor{red}{I am not sure how to interpret this theorem. First of all what is $P()$? }}
%Then, with confidence at least $1 - \delta$, 
%the true proportion $\alpha$ of $\mathcal{X}$ that satisfies the safety property satisfies the following bound 
Then under Assumption~\eqref{Assumption:hg},
\[
\Pr_{s \sim \mathcal{D}}\left[P(\mathrm{Reach}_K(s,\pi))\right] \geq \frac{|V|}{N} - \sqrt{\frac{1}{2N} \log \frac{2}{\delta}}
\]
with probability at least $1-\delta$.
%\footnote{\textcolor{red}{I am not sure what this part `with probability at least $1-\delta$' mean. Is it there by mistake? }}.
%, where $\epsilon = \sqrt{\frac{1}{2N} \log \frac{2}{\delta}}$.
\end{theorem}

\begin{proof}
Let $\alpha = \Pr_{s \sim \mathcal{D}}\left[P(\mathrm{Reach}_K(s,\pi,g))\right]$ and $\hat{\alpha} = \frac{|V|}{N}$.
%We model the safety verification process as a binomial experiment where each sampled point $s_i$ independently satisfies the safety property $P(s)$ with probability $\alpha$. The fraction of safe samples, $\hat{\alpha} = \frac{m}{N}$, serves as an estimator for $\alpha$. Since $N$ datapoints are independently sampled, according to the Hoeffding bound, we have
By the Chernoff-Hoeffding bound, $\Pr\left(|\hat{\alpha} - \alpha| \geq \epsilon \right) \leq 2e^{-2N\epsilon^2}$,
% \begin{equation*}
%     P\left(|\hat{\alpha} - \alpha| \geq \epsilon \right) \leq 2e^{-2N\epsilon^2},
% \end{equation*}
where the probability is with respect to datasets of $N$ initial states.
Letting $\delta = 2e^{-2N\epsilon^2}$, 
we obtain the confidence bound:
$\Pr\left(\alpha \geq \hat{\alpha} - \sqrt{\frac{1}{2N} \log \frac{2}{\delta}}\right) \geq 1 - \delta$.
% \begin{equation}
%     P\left(\alpha \geq \hat{\alpha} - \sqrt{\frac{1}{2N} \log \frac{2}{\delta}}\right) \geq 1 - \delta.
% \end{equation}
%where $ \epsilon   = \sqrt{\frac{1}{2N} \log \frac{2}{\delta}}$. 
%This means, with probability at least $1 - \delta$, the true proportion of the region that satisfies the safety property is at least $\frac{m}{N} - \epsilon$.
Finally, since by Theorem~\ref{T:approxvis}, $P(\mathrm{Reach}_K(s_0,\pi,g)) \Rightarrow P(\mathrm{Reach}_K(s_0,\pi))$, the result follows. 
%\textcolor{red}{[In the proof, it seems that $P$ means `probability' but in the theorem statement, $P$ refers to a predicate.]}
\end{proof}

The SPV framework above can thereby combine reachability over a finite sample of initial states to yield a rigorous tail bound guarantee for safety over a given safety horizon $K$.
%\footnote{\textcolor{red}{You may want to read/cite this paper. It proposes a sampling-based reachability analysis algorithm too \url{https://proceedings.mlr.press/v168/lew22a/lew22a.pdf}}}
This, of course, is for a given policy $\pi$.
In the next section, we turn to the main subject of our work: synthesizing control policies $\pi$ for the dynamical system~\eqref{E:dynsys} that achieve strong semi-probabilistic guarantees of this kind.
\section{Learning-Based Synthesis of Provably Safe Vision-Based Control}\label{sec:methods}

At the high level, our goal is to learn a policy $\pi$ which has a long safety horizon $K$ (that is, does not reach an unsafe state for any possible trajectory over as long a horizon $K$ as possible) with high probability $1-\delta$.
Suppose that $\pi_\theta$ is parametric with parameters $\theta$ (e.g., a neural network), and $K$ (i.e., the target safety horizon) is fixed.
Our goal is to maximize the probability that a trajectory is safe for at least $K$ steps, that is,
\begin{equation}
    \label{eq:safe_main}
    \max_\theta \Pr_{s \sim \mathcal{D}}\left[P(\mathrm{Reach}_K(s,\pi_\theta))\right].
\end{equation}
To make this practical, we can only rely on a finite sample of initial states, as well as make use of the cGAN $g$.
Consequently, the revised proxy objective is
\begin{equation}
    \label{eq:safe_main_proxy}
    \max_\theta \sum_i P(\mathrm{Reach}_K(s_i,\pi_\theta,g)).
\end{equation}
The previous section shows that this still enables rigorous semi-probabilistic verification.
%Next, we present our approach for approximately solving this problem.
Additionally, we consider a special case in which safety properties are tied to a scalar \emph{safety score} (for example, cross-track error in lane following).
In particular, let $\sigma(\mathrm{Reach}_K(s_i,\pi_\theta,g))$ be a safety score function over the reachable set, with $P(\mathcal{T})$ translating the safety score into a predicate (e.g., error exceeds a predefined threshold).
We assume that we can obtain differentiable bounds on $\sigma(\mathrm{Reach}_K(s_i,\pi_\theta,g))$ (e.g.,
%For example, this is possible 
if we use $\alpha,\beta$-CROWN~\cite{zhang2018efficient,xu2021fast,xu2020automatic}).
%Our main goal is to train a controller that can be efficiently verified as safe with high probability and over a longer reachability horizon. Given a well-trained generator $g$, our objective is to learn a policy $\pi$ such that:
%\begin{equation}
%\label{eq:safe_main}
%     \max_\theta \ p_1, p_2, \ldots, p_K
%\end{equation}
%where $p_k$ is the percentage of the initial states $\mathcal{S}_0$ can be verified safe for safety property $P$ for up to a $k$-step horizon.

\subsection{The Learning Framework}

% The main challange for training is that, because of the generator involves images, which is of high dimensionality and complexity, prior approach where aggregate the dynamics for multiple steps and perform back propogation~\cite{wu2024verified} does not apply here, due to the time and GPU memory constraints. Furthermore, due to the time and GPU memory consupmtion due to high dimensional images, train with all counterexamples is not feasible. In this work, we propose a more efficient framework for selecting and storing potentially unsafe examples. Furthermore, we designed a loss function where we leverage an empirical safe controller (e.g., PID controller) as an anchor, and combines supervised learning and verified safety loss, to make the training feasible with only one step lookahead, instead of $k$ steps as is in prior work.

A central challenge in synthesizing a provably safe policy $\pi_\theta$ in our setting arises from the involvement of high-dimensional images generated by the generator $g$ (as a proxy for the true perception model $h$), which serve as inputs to the controller. 
%Prior approaches, where the controller directly takes low-dimensional states as inputs~\cite{wu2024verified,emam2022safe,ganai2024iterative}, face significant scalability limitations and do not perform well in this setting (see Section \ref{sec:experiment} for details). 
%Methods such as aggregating dynamics over multiple steps and performing backpropagation, as well as training with all counterexamples, are impractical due to the significant time and memory consumption associated with high-dimensional images, which grows with batch size.
%Nevertheless, this prior work offers the overarching training scaffold upon which we build.
Our overall approach is as follows.
First, we begin with a pre-trained controller $\pi_{\hat{\theta}}$ that is empirically safe, which we also use as the \emph{anchor controller} to avoid sacrificing too much empirical performance as we train for safety verification.
Next, starting with $\hat{\theta}$, we train (or fine-tune) $\pi_\theta$ to minimize
\(
\sum_{i \in S_l} \mathcal{L}(s_i,\theta),
\)
where $S_l$ is a set of initial states $s_0$ used in training which evolves over training iterations $l$, and $\mathcal{L}(\theta)$ an appropriate loss function.
The central algorithm design questions thus amount to 1) the choice of the loss function, and 2) the problem of selecting data $S_l$ to use for training in each iteration, so as to ultimately obtain a provably (rather than merely empirically) safe policy.
We address these questions next.

%In this work, we design a loss function that leverages an empirically safe controller (e.g., a PID controller) as an anchor and combines supervised learning with a verified safety loss. 
%making the training feasible with a single-step lookaheadinstead of the $k$-step lookahead used in prior work.
%Furthermore, we propose a more efficient framework for selecting and storing potentially unsafe examples, which significantly reduces the number of datapoints (i.e., batch size) needed for training. Finally, we provide a theoretical guarantee for probabilistic verified safety of our approach.

% \subsection{Loss Function} 

% We propose a loss function that integrates both a supervised learning loss and safety loss as follows:
% \begin{align}
%     \mathcal{L}(s_i,\theta) = \lambda_1\mathcal{L}_{acc}(s_i,\theta) + \lambda_2\mathcal{L}_{safety}(s_i,\theta).
% \end{align}

% The supervised loss $\mathcal{L}_{acc}$ is the mean squared error loss between predictions of our controller and the anchor controller, and aims to preserve a strong empirical performance with respect to the original pre-trained anchor controller.

\subsection{Loss Function} 
We propose a loss function that integrates both a performance-preserving loss and safety loss as follows:
\begin{align}\label{eqn:loss_fn}
    \mathcal{L}(s_0, \theta) = \lambda_1 \mathcal{L}_{\text{perf}}(s_i,\theta)+\lambda_2 \mathcal{L}_{\text{safety}}(s_i, \theta).
\end{align}
The performance loss $\mathcal{L}_{\text{perf}}$ aims to preserve the empirical performance with respect to the pre-trained anchor controller $\pi_{\hat{\theta}}$.
Depending on the training method,
it can be (1) $\ell_2$ loss between predictions and ground truth for supervised training,
or (2) RL objective (e.g. policy gradient) to encourage high expected returns for RL training.

Turning next to the safety loss, recall that we assume that safety is quantified by a safety score function $\sigma(\mathrm{Reach}_K(s_i,\pi_\theta,g))$.
One candidate would simply be to use this score as part of the loss function.
However, this is impractical, as it is typically intractable to compute at scale and to the extent that it can be done, the tools for doing so are not differentiable.
% (e.g., using mixed-integer programming~\cite{tjeng2017evaluating}).
However, neural network verification techniques exist which compute \emph{differentiable sound upper and lower bounds on this quantity}, and these therefore make natural candidates to use in constructing a loss function.
More precisely, let $\underline{\sigma}(\mathrm{Reach}_K(s_i,\pi_\theta,g)) \le \sigma(\mathrm{Reach}_K(s_i,\pi_\theta,g)) \le \overline{\sigma}(\mathrm{Reach}_K(s_i,\pi_\theta,g))$ (i.e., $\underline{\sigma}(\cdot)$ is the lower and $\overline{\sigma}(\cdot)$ the upper bound on $\sigma(\cdot)$).
To simplify notation, we let $\sigma^{(i)}_K = \sigma(\mathrm{Reach}_K(s_i,\pi_\theta,g))$, with $\overline{\sigma}^{(i)}_K$ and $\underline{\sigma}^{(i)}_K$ the corresponding upper and lower bounds.
Then we define the safety loss as
\begin{equation}
\label{eq:veri_loss}
    \mathcal{L}_{safety}(s_i,\theta) = \frac{|\overline{\sigma}_K^{(i)}| + |\underline{\sigma}_K^{(i)}|}{K-1},
\end{equation}
% \begin{equation}
% \label{eq:veri_loss}
%     \mathcal{L}_{safety}(s_i,\theta) = \left(|\overline{\sigma}_K^{(i)}| + |\underline{\sigma}_K^{(i)}|\right)/\left(K-1\right),
% \end{equation}
which measures the rate of change of the reachable region. 

%To illustrate, suppose that we wish to achieve verified lane following where verification assures small lane following error $d_K$ over a horizon $K$, and where $d_0$ is the initial nominal error (lateral distance).
%Then
%\begin{align}\label{eq:veri_loss_lf}
%    \mathcal{L}_{safety} = \frac{|\overline{d_K} - d_0| + |\underline{d_K}-d_0|}{K-1},
%\end{align}
%where $\overline{d_K}$ and $\underline{d_K}$ are obtained using a differentiable verifier over a horizon $K$.
%\fi

%\textcolor{cyan}{where $d$ is the lateral distance. Intuitively, the verification loss measures averaged change of the lateral distance bounds when we perform a model rollout for $k$ steps starting from an initial state.}
% where $\mathcal{S}_k$ represents the reachable region for $k$-step reachability.
%Thus, the supervised learning loss ensures that the controller maintains strong performance, while the safety loss guides the controller toward satisfying verified safety constraints. 

% By benchmarking against an anchor controller through supervised learning, we can train using the verification loss with only a one-step lookahead. This approach avoids the computational infeasibility of incorporating full $k$-step reachability calculations, which would be prohibitively expensive due to the high dimensionality of image inputs.

\subsection{Adaptive Training Data}

Our adaptive training procedure performs gradient updates by sampling batches from an adaptive collection of training data $S_l$ which consists of two disjoint and fixed-size components: the set of random initial states $S_0$, and $S_A$, maintained as a priority queue, containing initial states for which safety is a challenge to satisfy. Specifically, when training starts, $S_A$ is empty, and we gradually populate $S_A$ during the warmup period by adding the $m$\% most challenging datapoints (in the sense detailed below) from each training batch to it. 
%We describe how we determine a datapoint is `worse' than others below. 
To ensure $S_0$ and $S_A$ remain disjoint, whenever a datapoint is added to $S_A$, it is also deleted from $S_0$, and we generate another datapoint uniformly randomly to add to $S_0$. 
%\junlin{I think the portion should also be a hyperparameter?} \textcolor{cyan}{it is a tunable parameter.}

% \junlin{(initial randomly sampled) $\xrightarrow{}$ (randomly sampled)? I am thinking every time we train we randomly sample some points, those random points are not fixed during training... i think, for each training cycle, 1,2,3,.. $K_i$, the randomly sampled one should keep the same. Next epoch we should sample those points again (still train against $K_i$). We can call this epoch?? and at the outer loop there is curriculum learning $K_1, K_2, .. .K_n =K$. ALso, we should discuss how we choose $K$ somewhere. }

We select datapoints to add to $S_A$ based on the rate of change in the safety margin $\sigma_i$ over an entire $K$-step trajectory. 
For example, a datapoint where the vehicle deviates from the center of the lane and drifts toward the margin at high speed is prioritized. 
This allows us to detect points that are likely to become unsafe ahead of time. The safety loss defined in Equation~\eqref{eq:veri_loss} can be directly used as this metric here.
If training fails to improve safety on datapoints in $S_A$, $S_A$ effectively becomes fixed once full. 
Otherwise, if safety improves or we encounter new datapoints that are less safe, $S_A$ adapts to reflect such changes.

After we have enough datapoints in $S_A$, future training batch $T$ with size $L$ consists of $\lfloor p \cdot L \rfloor$ datapoints from $S_A$ and the rest from $S_0$, where $p$ is a tunable parameter during training. The portion of datapoints from $S_0$ can be sampled uniformly. The portion of datapoints from $S_A$ are sampled following the Efraimidis \& Spirakis~\cite{efraimidis2006weighted} weighted sampling approach to prioritize datapoints that are more difficult (less safe) than others. Specifically, for each datapoint $i$ in $S_A$, we first calculate a safety parameter
\begin{align}
    \label{eq:sampling_score}
    \eta^{(i)} &= \mathcal{L}_{safety}^{(i)} \\
&+ \Big( \max(0, |\overline{\sigma}^{(i)}_K| - \beta) + \max(0, |\underline{\sigma}^{(i)}_K| - \beta) \Big),\nonumber
\end{align}
% \begin{align}
%     \label{eq:sampling_score}
%    & \eta^{(i)} = \frac{|\overline{\sigma}^{(i)}_K - d_0| + |\underline{\sigma}^{(i)}_K - d_0|}{K-1} \\
% + &\gamma \Big( \max(0, |\overline{\sigma}^{(i)}_K| - \beta) + \max(0, |\underline{\sigma}^{(i)}_K| - \beta) \Big),\nonumber
% \end{align}
where the second term further penalizes datapoints whose upper and lower bound is outside of the $[-\beta, \beta]$ region, where $\beta$ is a predefined safety threshold, and ensure such difficult-to-verify inputs are more likely to be sampled. We then assign a weight $w_i = e^{\alpha \cdot \eta^{(i)}}$ (where $\alpha$ is a hyperparameter) to each datapoint $i$ and normalize the weights such that $\sum_{i=1}^N w_i = 1$. 
%Here, $w_i$ reflects the degree of overlap with the unsafe set for unsafe states and the distance to the unsafe region for safe states. 
Finally we assign each element $x_i$ a key $k_i = w_i^{1/U_i}$, where $U_i \sim \text{Uniform}(0,1)$. Selecting the top $L$ elements with the highest $k_i$ values yields a weighted random sample without replacement, where each element $x_i$ is included with probability:
$P(x_i \in S) = w_i/\sum_{j \in \mathcal{B}} w_j$.

The full algorithm for data sampling is shown in Algorithm~\ref{algo_main:sample}. Line $3$ calculates how many datapoints to sample from the buffer based on a tunable parameter $p$; line $4-8$ computes the weights and keys following~\cite{efraimidis2006weighted}; line $9-11$ construct the training batch. This approach enables us to have an adaptive training set which maintains the controller's average performances while improving its verifiability with respect to datapoints that are more difficult.
%\vspace{-0.01in}

 \begin{algorithm}[!htbp]
\caption{Data Sampling Algorithm}
\label{algo_main:sample}
\begin{algorithmic}[1]
\STATE \textbf{Input:} 
$S_l = \mathcal{S}_0 \cup \mathcal{S}_A$, $\beta$, $\alpha$, $p$, $L$
\STATE \textbf{Output:} Sampled points $T$
\STATE $L_A \gets \lfloor p \cdot L \rfloor$; $L_0 \gets L - L_A$
\FOR{each datapoint $x_i \in \mathcal{S}_A$}
    \STATE Compute $\eta^{(i)}$ according to Equation~\ref{eq:sampling_score}
    \STATE Compute weight $w_i = e^{\alpha \cdot \eta^{(i)}}$
\ENDFOR
% \FOR{$i \in \mathcal{S}_A$}
%     \STATE normalize its weight and compute its key $k_i$
% \ENDFOR
\STATE Normalize $w_i$ and compute keys $k_i$ for all $x_i \in S_A$
% \vspace{1mm}
% \STATE $T_A = \{x_i\}_{i=1}^{L_A}, \quad x_i \sim \mathcal{S}_A$
% \vspace{1mm}
% \STATE $T_0 = \{x_i\}_{i=1}^{L_0}, \quad x_i \sim \mathcal{S}_0$
% \vspace{1mm}
\STATE $T_A \gets$ top $L_A$ elements in $S_A$ by $k_i$ values
\STATE $T_0 \xleftarrow{}$ uniformly sample $L_0$ points from $\mathcal{S}_0$
\STATE $T \gets T_{A} \cup T_0$
\STATE \textbf{Return} $T$
\end{algorithmic}
\end{algorithm}
%\vspace{-0.1in}

\subsection{Curriculum Learning} 

Finally, as in prior work~\cite{wu2024verified}, we use curriculum learning. Specifically, during training, we first target $K_i$-step verified safety. After a series of epochs, we progress to $K_{i+1}$-step verified safety, where $1 \leq K_1 < K_2 < \cdots < K_{n-1} < K_n = K$. This gradual increase in the verification horizon helps improve training stability and enables the controller to learn more complex safety constraints.
\section{Experiments}\label{sec:experiment}
\subsection{Experiment Setup}
We evaluate our approach in four settings: 
1) the X-Plane 11 Flight Simulator, 
2) the CARLA Simulator, 
3) a mini-city miniature urban physical autonomous driving platform with an F1Tenth racing car,
and 4) the Airsim Simulator for drone control.
We use two evaluation metrics: 
1) empirical performance and 
2) lower bound probability for safety guarantee as a function of $K$. 
We use three baselines for comparison: 
1) RESPO~\cite{ganai2024iterative}, a safe reinforcement learning framework using iterative reachability analysis; 
2) SAC-RCBF~\cite{emam2022safe}, which incorporates safety as a robust-control-barrier-function layer into training; 
and 3) VSRL~\cite{wu2024verified}, which guarantees finite-horizon safety by integrating incremental reachability verification into safe reinforcement learning.

Specifically, we first consider the autonomous aircraft taxiing problem using the \textbf{X-Plane 11 Flight Simulator}~\cite{xplane11sim}. 
For training the image generator, we collected 20,000 state-image pairs. 
Each sample consists of an image captured by a forward-facing camera and the corresponding state information (lateral offset $d$ ranging from $-10$ to $10$ meters and heading error $\theta$ between $-0.5$ and $0.5$ radians). 
Data was collected while the aircraft operated at a constant ground speed.

Our second set of experiments consider autonomous lane following using the 
\textbf{CARLA Simulator}~\cite{pmlr-v78-dosovitskiy17a} version 0.9.14. 
For training the image generator, we sample initial states consisting of the lateral distance $d$ from the lane center, the heading error $\theta$ relative to the lane direction, and the global coordinates $(x, y, z)$ within the CARLA map. 
The dataset includes trajectories with $d \in [-0.8, 0.8]$ meters and $\theta \in [-0.15, 0.15]$ radians. 
Final dataset contains 20,000 state-image pairs collected across different towns (maps) and environmental conditions.

Our third experiment extend the CARLA experiments to a \textbf{F1Tenth racing car}, an open-source 1:10 scale autonomous vehicle, for lane following in a \textbf{miniature city physical testbed}.
%\textcolor{red}{(Figure xx), a scaled city platform developed by WashU}.
The F1Tenth vehicle is equipped with a front-facing camera that provides visual input for lane following and can achieve scaled speeds comparable to full-scale autonomous vehicles. 
This platform enables us to evaluate the transferability of our training framework to real-world physical systems.
We collected $400$ images in the mini-city and manually annotated the state information for these. 
We then finetuned the image generator from the CARLA experiments on this dataset to obtain the image generator for the F1Tenth experiments.

The final problem we consider is drone navigation and obstacle avoidance in \textbf{Airsim Simulator}~\cite{airsim2017fsr}.
The drone is initialized at a fixed hovering position in front of an obstacle observable by its camera. 
The goal is to learn a smooth control policy while avoid collisions.
The image generator is trained on a dataset with $20,000$ samples. 
The state input consists of the drone's 3D position $(x,y,z)$ and orientation represented as a quaternion $(q_w,q_x,q_y,q_z)$. 
The dataset was collected while the drone was manually controlled by a remote controller. 

For all experiments, we pretrain an anchor controller first.
For the three path following experiments, the training is done by imitation learning from a tuned PID controller for the corresponding task, with learning rate of $0.0005$, batch size $256$, and $200$ epochs.
For the drone experiment, the anchor controller is trained using the standard PPO algorithm with $\gamma$ $0.99$, learning rate $0.0003$ for a total of $200000$ steps.

Then, during safety training,
we used supervised training for the path following experiments where we collected a dataset of $30,000$ samples of (state, action) pair. 
We train the controller for 100 epochs, with batch size $128$ and learing rate $0.0002$.
For the drone experiment, we integrated the safety loss directly into RL training using the PPO algorithm.
The controller is trained for an additional $200000$ steps with learning rate $0.0002$. 
Detailed hyperparameters and settings can be found in appendix~\ref{appendix:training}.

\begin{figure}[t]
  \centering
  \includegraphics[width=\columnwidth]{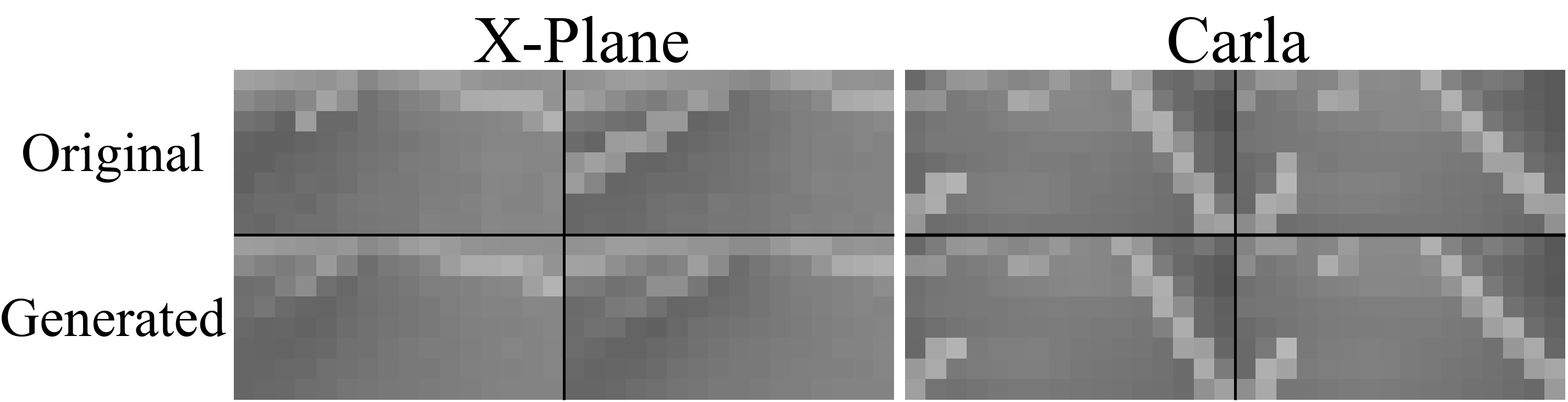}

  \caption{X-Plane and Carla generator illustrations.}
  \label{fig:genenrator_eval_carla_and_plane}
\end{figure}

\begin{figure}[t]
  \centering
  \includegraphics[width=0.5\columnwidth]{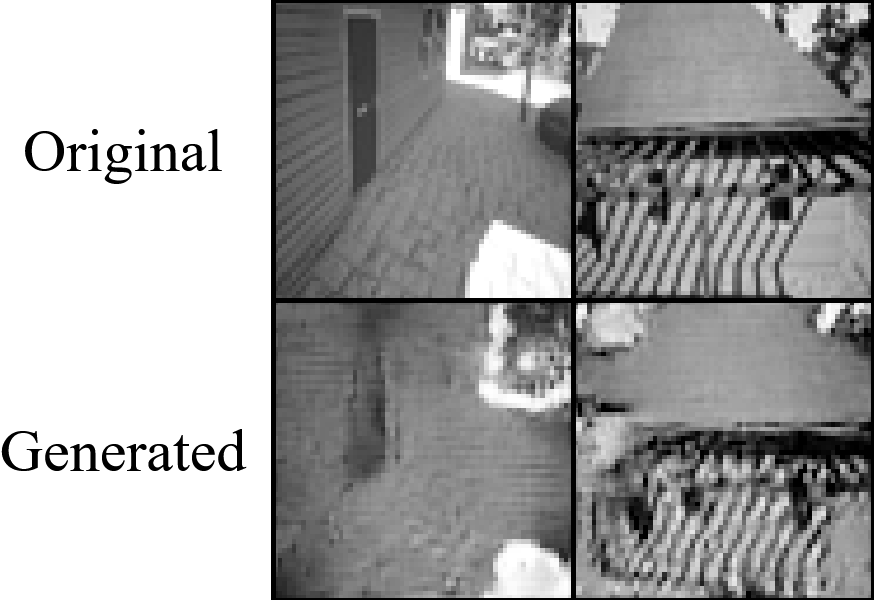}

  \caption{Drone generator illustration.}
  \label{fig:genenrator_eval_drone}
\end{figure}
\subsection{Image Generator Training and Evaluation}

The image generators are implemented as conditional GAN (cGAN) that maps from state information and a dimension $10$ latent vector to grayscale images; in ablations below we also consider diffusion models for this purpose.
The latent vector $z$ is sampled from a uniform distribution $\mathcal{U}(-1, 1)$, aiming to capture semantic variations that are not explicitly represented in our state information (e.g., lighting conditions, road textures, environmental elements). 
The generators output $8 \times 16$ images in the three path following experiments, and $64 \times 64$ images in the drone experiment.

For the path following experiments, 
we use a 4-layer fully connected neural network with hidden layers of 256 neurons each, using ReLU activation for the hidden layers and Tanh activation for the output layer. 
For the drone experiment, 
we use a more complex convolutional generator; its exact architecture is documented in appendix~\ref{appendix:training}. 
The discriminator in all experiments is a convolutional neural network. 
During training, 
we apply spectral normalization~\cite{miyato2018spectral} to all discriminator layers and orthogonal regularization~\cite{brock2016neural} to the generator loss function. 
Both generator and discriminator networks are initialized using orthogonal initialization and trained with batch size 128, learning rate 7e-4, for 100 epochs. 

Figure~\ref{fig:genenrator_eval_carla_and_plane} and Figure~\ref{fig:genenrator_eval_drone} show two ground truth images from the test set for each experimental setting alongside their corresponding generated images. 
Images from F1Tenth are excluded here as they are visually similar to the Carla images.
From visual inspections, the generators perform well for the Carla and Plane experiments, which is largely due to the low-dimensional nature of the images in these settings.
For the Drone experiment where we have higher resolution images, on the other hand, 
while the generated images are less precise, they still capture the essential semantic features needed for effective controller learning.
%Further optimizing generator performance while balancing computational efficiency during safety verification falls outside the scope of this work.

\begin{figure}[ht]
    \centering
    \includegraphics[width=\columnwidth]{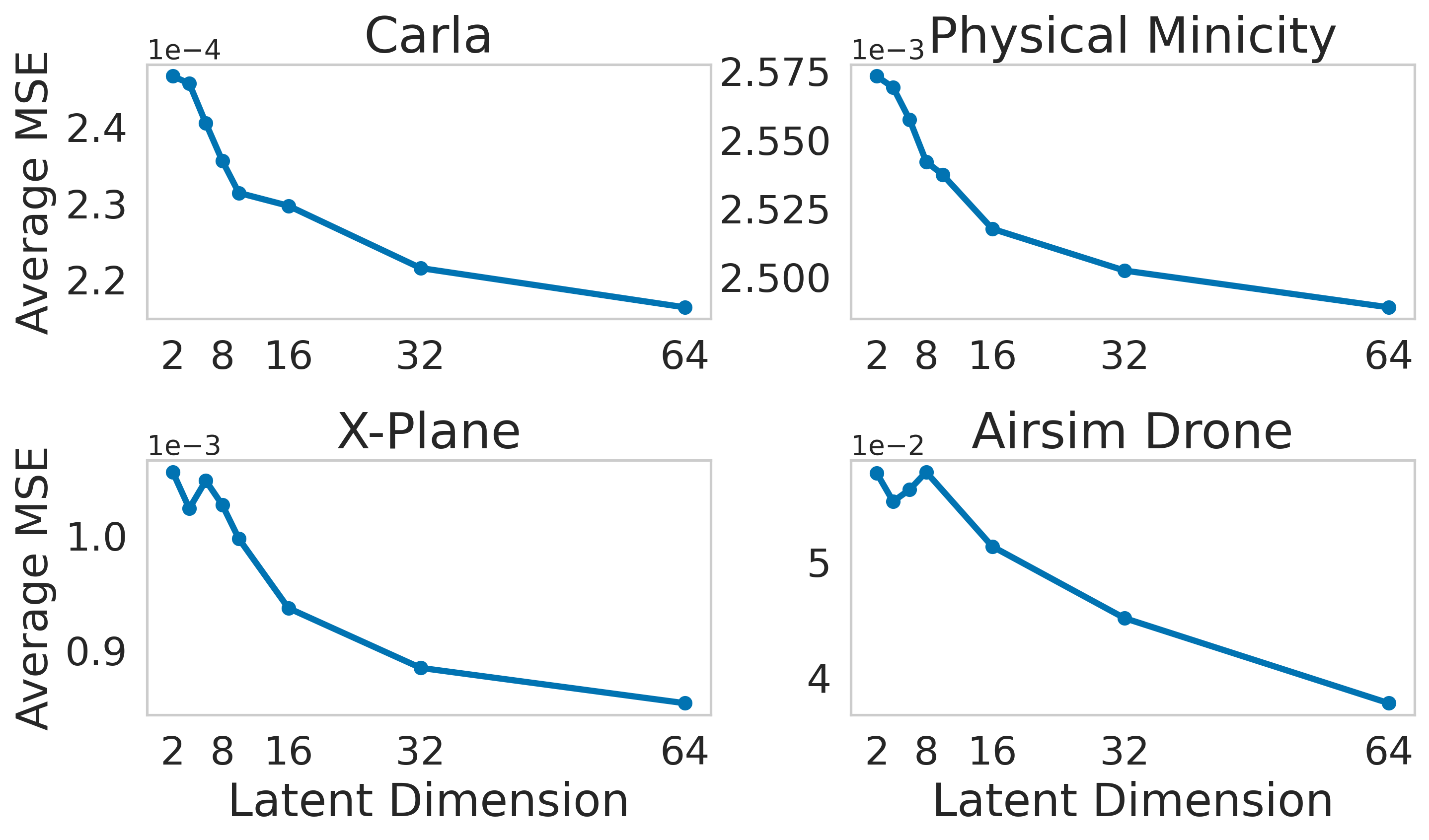}
    \caption{Empirical validation of Assumption~\ref{Assumption:hg}.}
    \label{fig:assumption1}
\end{figure}
An important ingredient in our approach is Assumption~\ref{Assumption:hg}, that for true image inputs we can find a latent representation $z$ as input to the generator that results in its close approximation. 
We validate this empirically for $8$ choices of latent dimensions: $2,4,6,8,10,16,32,64$. 
We evaluate the generators on a separate test set: for each ground-truth image, we perform gradient-based optimization to find the best latent variable $z$ that minimizes the mean squared error (MSE).
Notably, since this is not guaranteed to find a globally optimal $z$, our results provide a pessimistic evaluation (i.e., it is possible that such a $z$ exists even if we fail to find it using gradient-based optimization).

As shown in Figure~\ref{fig:assumption1},
as latent dimension increases, the averaged MSE rapidly decreases, i.e., the generator $g$ better approximates the unknown $h$ in all four experiment settings we consider in this work. 
Moreover, we can find effective approximations of the true image inputs for even relatively small latent dimension.

\begin{figure*}[h]
      \centering
%       \framebox{\parbox{3in}{We suggest that you use a text box to insert a graphic (which is ideally a 300 dpi TIFF or EPS file, with all fonts embedded) because, in an document, this method is somewhat more stable than directly inserting a picture.
% }}
      %\parbox{3.5in}
      {\includegraphics[width=0.85\linewidth]
      %[scale=0.14]
{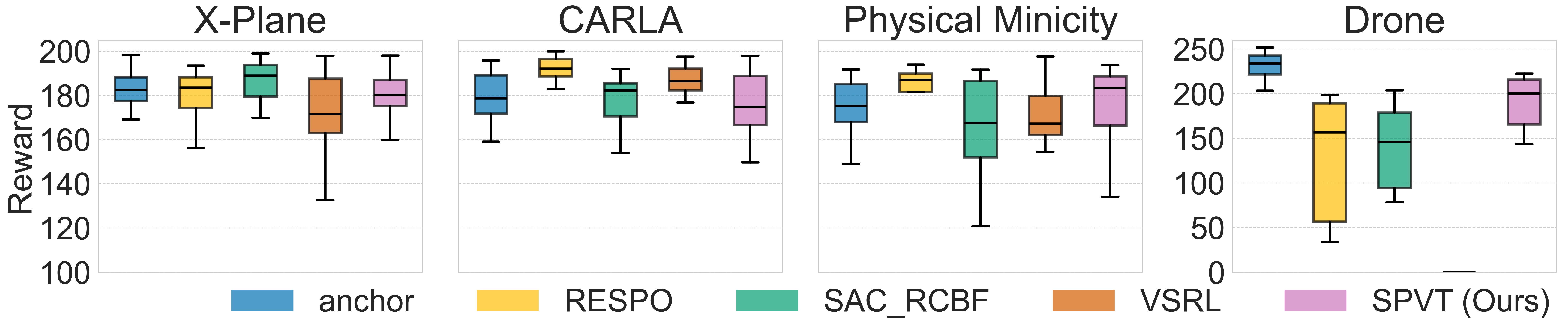}
}
      %\includegraphics[scale=1.0]{figurefile}
      %\vspace{-0.1in}
      \caption{Empirical performance comparison of controllers. Box plots show the distribution of episode rewards over 100 evaluation episodes.}
      \label{fig:controller_rewards}
   \end{figure*}

\begin{figure*}[h]
  \centering
%       \framebox{\parbox{3in}{We suggest that you use a text box to insert a graphic (which is ideally a 300 dpi TIFF or EPS file, with all fonts embedded) because, in an document, this method is somewhat more stable than directly inserting a picture.
% }}
  %\parbox{3.5in}
  {
  \includegraphics[width=0.85\linewidth]
  %[scale=0.14]
  {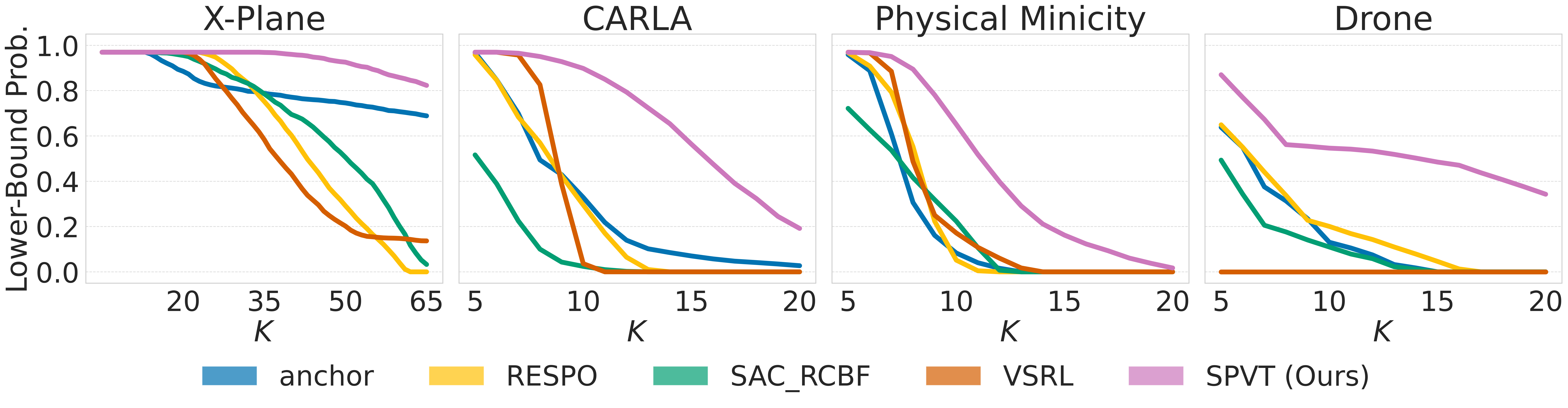}
}
  \caption{\textit{Semi-Probabilistic Verification (SPV)} results: $x$-axis marks the target verification trajectory length ($K$). $y$-axis is the lower bound safety probability.}
  \label{fig:spv_results}
\end{figure*}

\subsection{Results}

%\textbf{Empirical Performance:} 
\subsubsection{Empirical Performance}

To evaluate the controllers' empirical performance, we simulate trajectories starting from $100$ random initial states for $200$ steps. 
We use the undiscounted cumulative rewards $\sum r_t$ as the evaluation metric.
For the path following tasks (X-Plane, Carla, and Physical Minicity), 
the reward at each step $t$ is defined as $r_t=1-\min(1, |d_t|/\beta)$, where $d_t$ is the cross-track error at step $t$ and $\beta$ is a predefined safety threshold.
We set $\beta=1.0$m for Carla and Physical Minicity evaluations, and $\beta=12.5$m for X-Plane evaluations.

For drone control in Airsim, 
we used the same reward function in both training and evaluation.
The reward includes components for forward velocity, attitude stability, and control smoothness.
The complete reward function is provided in appendix~\ref{appendix:exp_drone}.

As shown in Figure~\ref{fig:controller_rewards}, the proposed \textsc{SPVT} approach maintains strong empirical performance across all experimental settings. 
Notably, in the Drone experiment, which has the highest-dimensional image input, \textsc{SPVT} demonstrates significantly 
superior performance compared to all safe training baselines. This result suggests that \textsc{SPVT} is the only scalable safe training method capable of handling high-dimensional image inputs. 
It is worth noting that VSRL training fails entirely in the Drone experiment because VSRL requires formal verification over the controller's complete input space, which becomes computationally infeasible for high-dimensional image data.

\subsubsection{Safety Guarantee}

For the path following tasks,
the safety property we consider is the vehicle not leaving the current lane, equivalently, $|d_i| \leq \beta$ for $0 \leq i \leq K$.
For the drone experiment,
safety properties include altitude bounds ($z \in [0.2, 15]$ meters) and attitude limits (roll and pitch angles $\leq 1.2$ radians).

To obtain the lower-bound safety probability under our semi-probabilistic verification (SPV) framework, we collected a dataset of $2000$ initial states sampled i.i.d. from the initial state distribution. 
We used $\alpha,\beta$-CROWN~\cite{zhang2018efficient, xu2020automatic, xu2021fast} to verify the model iteratively.
At each step,
$\alpha,\beta$-CROWN computes bounds on the model output, which, after dynamics calculations, gives the reachable state regions at the next timestep.
We then check whether this state region intersects with unsafe states. 
We iterate this forward reachability analysis for $K$ steps, 
obtaining empirical verified safety result, i.e. the proportion of initial states in our dataset that are verified safe at each step.
These empirical results are then used in Theorem~\ref{thm:hoeffding_bound} to compute formal lower-bound safety probabilities over the entire initial state space.
We set the confidence parameter $\delta=0.05$ for all experiments. 

Figure~\ref{fig:spv_results} shows the resulting safety probability lower bounds as a function of horizon $K$. Compared to all baseline methods, the proposed \textsc{SPVT} approach significantly increases these probabilistic safety bounds. 

\subsubsection{Ablations}

\begin{figure}[ht]
    \centering
    \includegraphics[width=0.85\columnwidth]{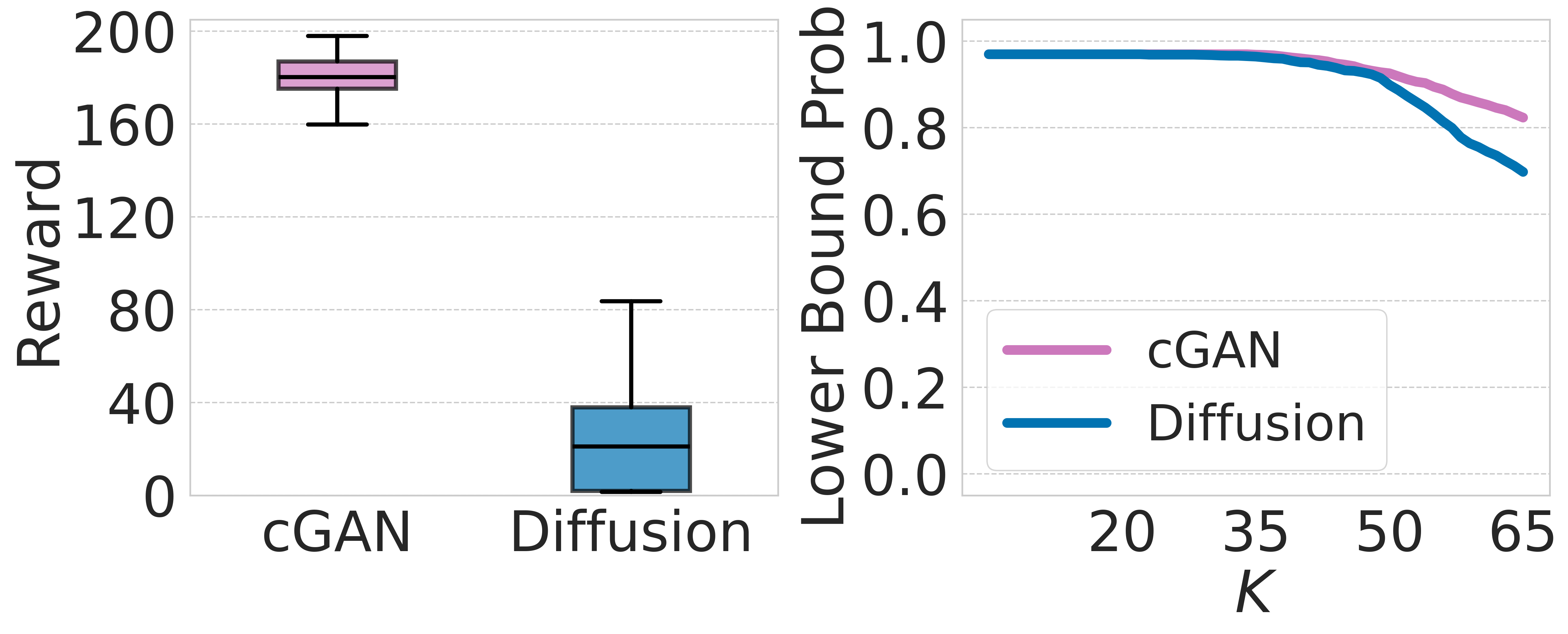}
    \caption{Ablation study comparing cGAN vs. diffusion-based state-to-image mapping in X-Plane. Left: controller performance over 100 episodes. Right: SPV verified safety probability bounds.}
    \label{fig:ablation_plane}
\end{figure}

A natural alternative to our cGAN-based state-to-image mapping is to use diffusion-based image generation.
In this case, $g(s,z) = z - \phi(x,z)$, where the diffusion model $\phi(x,z)$ predicts the noise in the randomly generated input.
To provide finer control over the dimension of $z$, we make use of an autoencoder $\psi(z)$ that takes randomly generated vector $z$ of a given dimension as an input, and outputs a decoded noisy image $z' = \psi(z)$ with the dimension of the image, so that $g(s,z)=\psi(z) - \phi(x,\psi(z))$.
To evaluate this design choice, we implement a diffusion-based variant of \textsc{SPVT} and compare it against our cGAN approach on both controller performance and safety verification.
Figure~\ref{fig:ablation_plane} shows results in the X-Plane environment.
We find that controller trainied in the cGAN setup achieve higher rewards and maintain better safety probability bounds.
% We refer the readers to the experiment section and appendix in the extended version for more results and detailed analysis.
Ablation results in other environments, along with more detailed analysis of the observed difference, are provided in Appendix~\ref{appendix:ablations}.

\section{Related Work}\label{sec:relatedWork}
%\textcolor{red}{mainly borrowed and modified based on VSRL paper, added vision components and citations. But we shouldn't start with safeRL, not sure how to start this section...} \junlin{i think we can trim this section a bit to save space}
%\junlin{the paper is not about safe RL. i think }

Formally verifying the safety of vision-based controllers is extremely challenging. 
% On one end, probabilistic guarantees (e.g., sampling or Monte Carlo approaches) relying on statistical assumptions are practical but not complete~\cite{lew2022simple, weng2019proven, webb2018statistical, closed_loop_analysis_of_vision_based_control_systems_case_study_Mangal_CAV_2023}.
Traditional verification tools~\cite{katz2017reluplex, huang2017safety, ehlers2017formal, sinha2017certifying} are too computationally demanding for this purpose.
Recent progress has enabled verification of larger neural networks through over-approximation \cite{zhang2018efficient, gowal2018effectiveness, xu2021fast}, and abstract interpretation \cite{gehr2018ai2, singh2019abstract, katz2019marabou}. 
However, verifying vision-based controllers remains challenging due to high dimensionality, as well as the difficulty of defining safety properties in the image space.
%since it is inherently less straightforward to define safety properties in the image space.

A practical approach to verify vision-based controllers is by approximating the perception module.
This reduces verification complexity by projecting the problem from image space to lower-dimensional state space~\cite{katz2022verification, cai2024scalable, hsieh2022verifying}.
These methods, however, only focus on post-hoc verification and cannot be integrated into the training process for learning a safer controller.
For a more comprehensive survey on the verification of vision-based controllers, see~\cite{vision_based_control_verification_survey_2024}.
% \hussein{For a more comprehensive survey on the verification of vision-based controllers, check~\cite{vision_based_control_verification_survey_2024}.}

Safe reinforcement learning encodes safety based on the constrained Markov decision processes~\cite{altman2021constrained} but it is often too soft to enforce strict safety constraints~\cite{wang2023}. On the other hand, control-theoretic methods such as barrier functions~\cite{dawson2022safe, emam2022safe, barriernet_2023}, and reachability analysis~\cite{yu2022reachability, ganai2024iterative} are more powerful but they have limited scalability to high-dimensional system such as vision-based controllers. 
Recently, some works extend these methods to the image space~\cite{V-CBF_Abdi_ICRA2023, NeRF_based_CBF_ICRA_2023, yuxuan_2024_icra, Ihab_2024_L4DC}. 
However, they function more like a safety check around the controller, giving yes/no answers for passed in reference control, without providing formal safety guarantees about the controller itself.

% \cite{NeRF_based_CBF_ICRA_2023,V-CBF_Abdi_ICRA2023,safe_vision_based_rl_L4DC_2024,yuxuan_2024_icra,Ihab_2024_L4DC,closed_loop_analysis_of_vision_based_control_systems_case_study_Mangal_CAV_2023,vision_based_control_verification_survey_2024}

\section{Conclusion and Limitations}\label{sec:conclusion}

We introduced a semi-probabilistic verification framework for efficiently training and verifying vision-based neural network controllers. Our method models the perceptual mapping from state to image with a conditional generator and uses distribution-free tail bounds to get safety guarantees over the entire initial state space.
%while only requiring to verify a sampled set of states. 
We designed a differentiable proxy to the safety verification objective under the SPV framework that can be directly incorporated into gradient-based training, and an adaptive training set that prioritizes states for which safety property is difficult to verify. 
While our experiments demonstrate the efficacy of our approach, many limitations remain.
For the moment, our controllers and image generator require images to be grayscale and relatively low resolution, and latent dimension of the generator is relatively small.
Further research is needed to handle high-resolution visual inputs, as well as to extend to multi-modal sensing.

%We evaluate our approach on neural network controllers for path following in X-Plane, CARLA, and a visually-rich physical minicity environments. Experiments show that our controllers have significant improved verified safety over state-of-the-art safe training baselines while maintaining strong empirical control performances. 

% \section*{Acknowledgements}
% This research was partially supported by 
% NSF (IIS-2214141, CCF-2403758), 
% ARO (W911NF-25-1-0059), 
% ONR (N00014-24-1-2663), 
% Foresight Institute, 
% and Amazon.

\section*{Acknowledgements}
This work was partially supported by 
NSF (IIS-2214141, CCF-2403758), 
ARO (W911NF-25-1-0059), 
ONR (N00014-24-1-2663), 
Foresight Institute, 
and Amazon.

\bibliography{our_ref}

% \clearpage
% \input{ReproducibilityChecklist}

\clearpage
\appendix
\setcounter{secnumdepth}{1}  % Ensure appendix sections are still numbered
\renewcommand{\thesection}{\Alph{section}}
\setcounter{section}{0}
\section{Experiment Details}

\subsection{Path Following Experiments}

\paragraph{Dynamics}

For X-Plane, Carla, and Physical Minicity, we use a discrete-bicycle model for vehicle dynamics~\cite{kong2015kinematic}:
\begin{align*}
\dot{x} &= v \cos(\theta + \beta) \\
\dot{y} &= v \sin(\theta + \beta) \\
\dot{\theta} &= \frac{v}{l_r}\sin(\beta) \\
\dot{v} &= a \\
\beta &= \tan^{-1}\left(\frac{l_r}{l_f + l_r}\tan(\delta_f)\right)
\end{align*}
where $x$ and $y$ denote the vehicle's position coordinates, $\theta$ denotes heading angle relative to the path, $v$ is the vehicle speed, $a$ is the acceleration, $\delta_f$ is the front wheel steering angle (or equivalent control input), $\beta$ is the side slip angle, and $l_f$ and $l_r$ are distances from the center of gravity to front and rear axles respectively.

\subsection{Drone Control in Airsim}\label{appendix:exp_drone}

\paragraph{Dynamics}

We implemented the standard dynamics followng quadrotor models. 
The system state is represented as a 12-dimensional vector
$\mathbf{x} = [\mathbf{p}, \boldsymbol{\phi}, \mathbf{v}, \boldsymbol{\omega}]^T$,
where $\mathbf{p} = [x,y,z]^T$ is position, $\boldsymbol{\phi} = [\phi, \theta, \psi]^T$ are Euler angles (roll, pitch, yaw), $\mathbf{v} = [v_x, v_y, v_z]^T$ is linear velocity, and $\boldsymbol{\omega} = [\omega_x, \omega_y, \omega_z]^T$ is angular velocity.

The continuous-time dynamics are given by:
\begin{align*}
\dot{\mathbf{p}} &= \mathbf{v} \\
\dot{\mathbf{v}} &= g\mathbf{e}_3 + \frac{1}{m}\mathbf{R}(\boldsymbol{\phi})\mathbf{f} \\
\dot{\boldsymbol{\phi}} &= \mathbf{T}(\boldsymbol{\phi})\boldsymbol{\omega} \\
\dot{\boldsymbol{\omega}} &= \mathbf{I}^{-1}(\boldsymbol{\tau} - \boldsymbol{\omega} \times \mathbf{I}\boldsymbol{\omega})
\end{align*}
where $\mathbf{R}(\boldsymbol{\phi})$ is the rotation matrix from body to world frame, $\mathbf{f} = [0,0,f_{total}]^T$ is the thrust vector in body frame, $\boldsymbol{\tau}$ is the torque vector, and $\mathbf{T}(\boldsymbol{\phi})$ is the transformation matrix from body angular rates to Euler angle rates.

The system parameters are: 
mass $m = 0.468$ kg, inertia matrix $\mathbf{I} = \text{diag}(4.9, 4.9, 8.8) \times 10^{-3}$kg$\cdot$m$^2$, gravitational acceleration $g = 9.81$ m/s², and maximum thrust $f_{max} = 2.5mg$. The simulation uses a discrete timestep $\Delta t = 0.02$ seconds.

\paragraph{Reward function}

We designed the reward:
\begin{align*}
    r_t = r_{\text{forward}}+r_{\text{alive}}+r_{\text{attitude}}+r_{\text{angular}}
\end{align*}
where
\begin{itemize}
    \item $r_{\text{forward}}=\text{clip}(v_{\text{forward}},0,5)$ where $v_{\text{forward}}$ is the drone's forward velocity
    \item $r_{\text{alive}}=1.0$ for each timestep without termination signal
    \item $r_{\text{attitude}}=-0.1(\phi^2+\theta^2)$ to discourage excessive roll/pitch
    \item $r_{\text{angular}}=-0.05|\boldsymbol{\omega}|^2$ to discourage jerky controls
\end{itemize}

Termination conditions include
(1) collision,
(2) excessive pitch or roll with threshold set at $1.2$ radians,
(3) altitude outside range $[0.2,15]$ meters.
(3) ensures that the agent does not learn to exploit the alive bonus by choosing trivial actions to fly excessively high. Additionally, altitudes that are too high pose challenges for our generator model, as it was not trained on such data.

\section{Training Details}\label{appendix:training}

\subsection{Image Datasets}

\paragraph{X-Plane:}
The image generator is trained on a dataset of $20,000$ samples collected from the X-Plane simulator~\cite{xplane11sim}.
The generator takes as input a latent variable $z$ of dimension $10$ and state variable $x=(d,\theta)$ where $d$ is cross track error and $\theta$ is heading error. 
All images are converted to grayscale and downsampled to $8 \times 16$ pixels.

\paragraph{Carla:}
The image generator is trained on a dataset of $20,000$ samples collected from the Carla simulator~\cite{pmlr-v78-dosovitskiy17a} across built-in map $1-4$. 
Its input consists of a latent variable $z$ if dimension $10$ and state variable $x=(d,\theta,x,y,z)$ where $d$ is cross track error, $\theta$ is heading error, and $x,y,z$ are the global coordinates of the vehicle within the Carla simulator.
The images are converted to grayscale and downsampled to $8 \times 16$.

\paragraph{Physical Minicty F1Tenth:}
Image generator is finetuned from the Carla image generator using $400$ images collected from our minicity testbed.
We manually annotated these images with corresponding state information.
The images are also converted grayscale and downsampled to $8 \times 16$ pixels.

\paragraph{Airsim:}
Image generator is trained on a dataset of $20,000$ samples collected from the Airsim simulator~\cite{airsim2017fsr}, using the AirSimNH map (an urban neighborhood environment). 
Its input consists of a latent variable $z$ of dimension $10$ and state variable $x=(x,y,z,q_w,q_x,q_y,q_z)$ where the $(x,y,z)$ represents the drone's 3D position and $(q_w, q_x, q_y, q_z)$ represents its orientation in quaternion format. 
The images are converted to grayscale and downsampled to $64 \times 64$ pixels.

\subsection{Anchor Controller Pretraining}

\paragraph{Path Following Experiments:}
In the X-Plane amd Carla experiments,
the anchor controllers were pretrained on a pre-collected dataset of $30,000$ samples.
Each sample is an image-action pair.
For X-Plane, we first extract the cross track error and heading error from the simulator, then apply the proportional control law $\phi = -0.74d - 0.44\theta$ to compute the ground-truth action.
For Carla, ground-truth actions are extracted directly from the simulator while the vehicle operates in autonomous drive mode. 
In the Physical Minicity experiment, the anchor controller is directly adopted from the Carla simulator, as collecting a sufficiently large dataset for training from scratch in the physical environment is challenging.

\paragraph{Airsim Drone Control Experiment:}
The anchor controller is trained using the PPO algorithm within the Airsim simulator on the map AirSimNH. 
The reward function is defined in Appendix~\ref{appendix:exp_drone}, and detailed hyperparameters are documented in Appendix~\ref{appendix:(hyper)parameters}.

\subsection{Hyperparameters}\label{appendix:(hyper)parameters}

\subsubsection{Generator Training}
Table~\ref{tab:gen_parameters} contains the hyperparameters and settings for training our cGAN generator in our experiments.

\begin{table}[!t]
\begin{tabular}{c|c}
\textbf{Hyperparameters}                  & Settings           \\ \hline
Learning Rate                             & $7 \times 10^{-4} \rightarrow 5 \times 10^{-5}$ \\
Learning Rate Schedular & Cosine \\
Batch Size                                & $128$              \\
Epochs                                    & $100$              \\
Latent Dimension                          & $10$               \\
Hidden Layer Activation Function          & ReLU               \\
Output Activation Function                & Tanh               \\
$\ell_1$ Loss Weight                      & $10$               \\
Adversarial Loss Weight                   & $1.0$              \\
Orthogonal Regularization Weight          & $1 \times 10^{-4}$ \\ \hline
\textbf{Path Following Specific Settings} &                    \\
Network Architecture                      & MLP cGAN           \\
Hidden Layers                             & $(256) \times 4$   \\
Image Resolution & $8 \times 16$ \\
\hline
\textbf{Airsim Drone Specific Settings}   &                    \\
Network Architecture                      & CNN cGAN           \\
Image Resolution & $64 \times 64$ \\
Base \# of Channels                       & $24$               \\
\# of Residual Blocks                     & $2$                \\
Condition Embedding Dimension             & $128$              \\
Weight Initialization                     & Xavier Uniform     \\
Normalization                             & BatchNorm2d        \\
Upsampling Method                         & ConvTranspose2d   
\end{tabular}
\caption{Generator Training Hyperparameters and Settings}
\label{tab:gen_parameters}
\end{table}

\subsubsection{Anchor Controller Training}
Table~\ref{tab:anchor_training_parameters} documents the hyperparameters and settings for training our anchor controllers.

\begin{table}[!t]
\begin{tabular}{c|c}
\textbf{Hyperparameters}         & Settings                \\ \hline
Hidden Layer Activation Function & ReLU                    \\
Output Activation Function       & Tanh                    \\
Architecture                     & MLP                     \\ \hline
\textbf{Path Following}          &                         \\
Learning Rate                    & $5 \times 10^{-4}$      \\
Batch Size                       & $256$                   \\
Epochs                           & $200$                   \\
Loss Function                    & Supervised MSE          \\
Hidden Layers                    & $(128,64,32,8)$         \\ \hline
\textbf{Airsim Drone}            &                         \\
Learning Rate                    & $3 \times 10^{-4}$      \\
Warmup Steps                     & $3000$                  \\
Steps per Update                 & $2000$                  \\
Batch Size                       & $128$                   \\
\# of Epochs per Update          & $80$                    \\
Total Steps                      & $200,000$               \\
PPO Clip Ratio                   & $0.2$                   \\
GAE Lambda                       & $0.95$                  \\
Discount Factor                  & $0.99$                  \\
Hidden Layers                    & $((512) \times 3, 100)$
\end{tabular}
\caption{Anchor controller training hyperparameters and settings}
\label{tab:anchor_training_parameters}
\end{table}

\subsubsection{\textsc{SPVT}}
Table~\ref{tab:spvt_training_parameters} contains hyperparameters and settings for \textsc{SPVT} training. 
Note that for Drone training, verification of the larger generator plus controller model is computationally demanding, 
so each update only uses a mini-batch of size $32$ to compute safety objective.

\begin{table}[!t]
\begin{tabular}{c|c}
\textbf{Hyperparameters}              & Settings                                        \\ \hline
$\lambda_1$ (Performance Loss Weight) & $0.25$                                          \\
$\lambda_2$ (Safety Loss Weight)      & $1.0$                                           \\
Curriculum Starting $K$               & $4$                                             \\
Curriculum Final $K$                  & $10$                                            \\
$S_A$ (Adversarial Training Set) Size & $4000$                                          \\ 
$p$: Adversarial Sample Proportions & start at $0.5$
    \\
$p$ Step Size & $0.05$
\\
\hline
\textbf{Path Following}               &                                                 \\
Learning Rate                         & $8 \times 10^{-5} \rightarrow 1 \times 10^{-5}$ \\
Learning Rate Scheduler               & Cosine                                          \\
Batch Size                            & $256$                                           \\
Epochs                                & $100$                                           \\ \hline
\textbf{Airsim Drone}                 &                                                 \\
Learning Rate                         & $3 \times 10^{-4}$                              \\
Steps per Update                      & $1000$                                          \\
Batch Size                            & $128$                                           \\
\# of Epochs per Update               & $20$                                            \\
Total Steps                           & $100,000$                                       \\
Safety Loss Batch Size                & $32$                                            \\
PPO Clip Ratio                        & $0.2$                                           \\
GAE Lambda                            & $0.95$                                          \\
Discount Factor                       & $0.99$                                         
\end{tabular}
\caption{\textsc{SPVT} training hyperparameters and settings}
\label{tab:spvt_training_parameters}
\end{table}

\subsubsection{Compute Resources}
All generator trainings, anchor controller trainings, and \textsc{SPVT} training in path following experiments were ran on NVIDIA GeForce RTX 3090 GPU or NVIDDIA GeForce 4080 Super GPU. \textsc{SPVT} training for drone experiment was ran on NVIDIA A100 80GB GPU. 

\section{More Ablations}\label{appendix:ablations}

In this section, 
we first show additional ablations on generator architecture choice.
Then, we show ablations on two other design choices:
the adaptive training data and curriculum learning.

\begin{figure}[h]
    \centering
    \includegraphics[width=0.85\columnwidth]{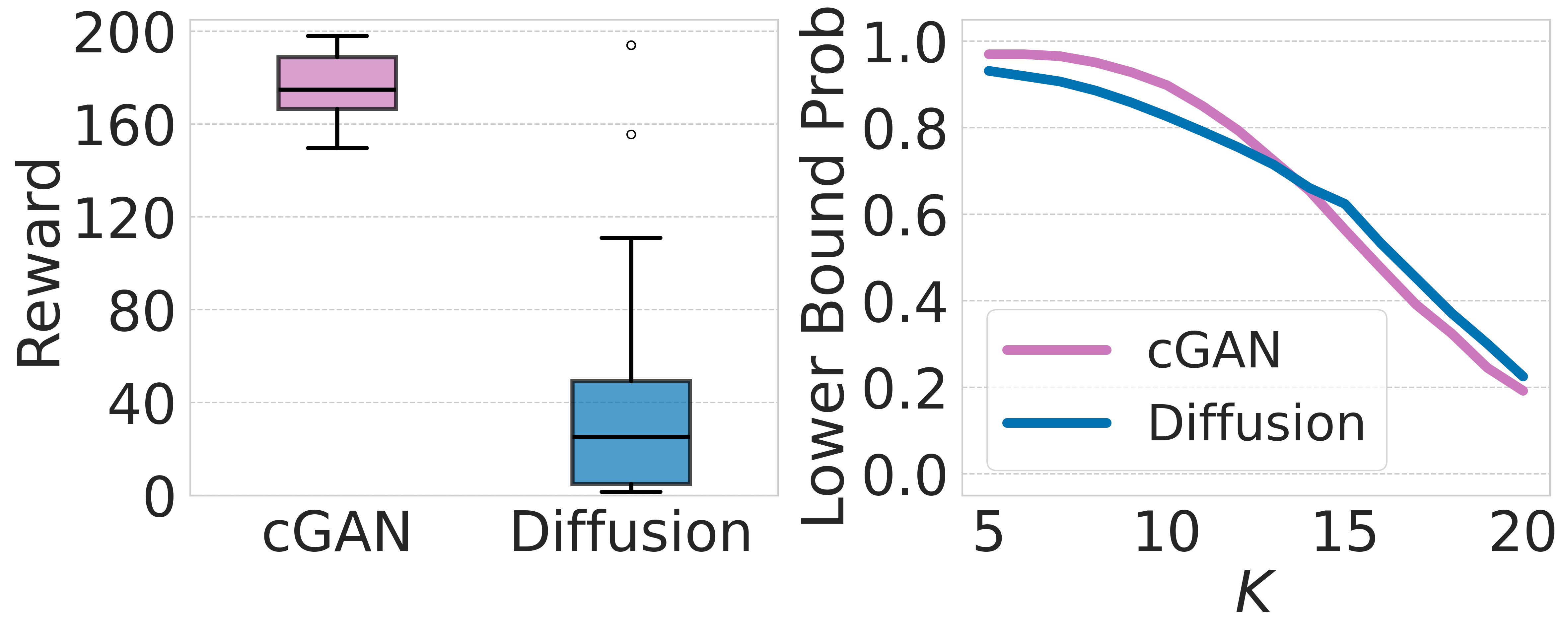}
    \caption{Ablation study comparing cGAN vs. diffusion-based state-to-image mapping in Carla. Left: controller performance over 100 episodes. Right: SPV verified safety probability bounds.}
    \label{fig:ablation_carla}
\end{figure}

\begin{figure}[!t]
    \centering
    \includegraphics[width=0.85\columnwidth]{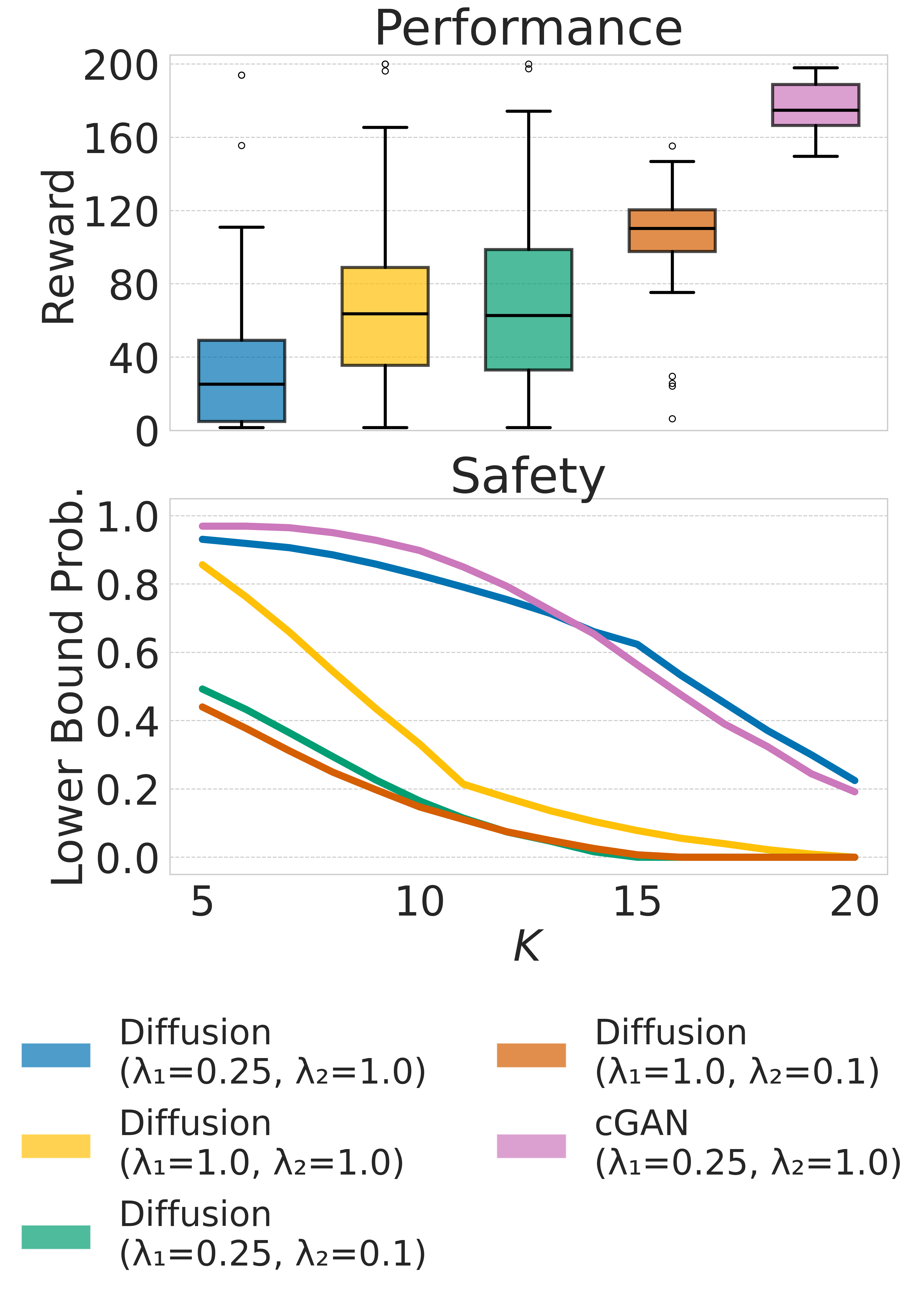}
    \caption{Analysis of hyperparameter effects for diffusion-based SPVT controllers in Carla. Top: Controller performance. Bottom: SPV safety probability lower bound.}
    \label{fig:ablation_compare_hyperparameter}
\end{figure}

\subsection{Generator Architecture: More Details}

Figure~\ref{fig:ablation_carla} shows results from the Carla environment, with similar findings to our ablation experiments from the X-Plane environment.
We observe that controller performance degrades significantly under the diffusion-based setup, though safety probability bounds are more comparable to those achieved in the cGAN approach.

Our method offers a convenient way to address this safety-utility tradeoff by adjusting the hyperparameters in Equation~\ref{eqn:loss_fn}.
Specifically, increasing $\lambda_1$ or decreasing $\lambda_2$ to prioritize maintaining performance.
However, figure~\ref{fig:ablation_compare_hyperparameter} demonstrates the challenge of hyperparameter tuning in diffusion-based SPVT.

We observe that increasing $\lambda_1$ alone while maintaining $\lambda_2$ at the same level used in our cGAN setup barely improves the controller performance for the diffusion-based approach (yellow box/line).
Furthermore, when $\lambda_2$ is simultaneously decreased, controller starts to recover good performances, but it comes at the cost of significant degradation in safety (orange box/line).
The green plot shows results of decreasing $\lambda_2$ alone, we see that it has minimal impact on controller performance while significantly compromises safety.

We hypothesize that the fundamental limitation of diffusion-based image generators within our SPVT framework comes from the meaningless output bound produced by even the state-of-the-art neural network verification tool.
Although we treat generator and controller as one unified enitty during our training process,
generator parameters are frozen to preserve the learned state to image mapping.
Hence, the verification process required to compute our safety loss reduces to verifying the controller on the output bounds (bounded image) of the generator model.
However, when using diffusion-based image generator, bounds must be computed at each denoising step, leading to compounded accumulation of uncertainty.
This results in extremely loose bounds on teh output image, which offers no meaningful constraints for safety checking.
On the other hand, the cGAN generator outputs the image directly in a single forward pass, so we are able to obtain meaningfully tight bounds on the output image. 

\begin{figure}[h]
    \centering
    \includegraphics[width=0.85\columnwidth]{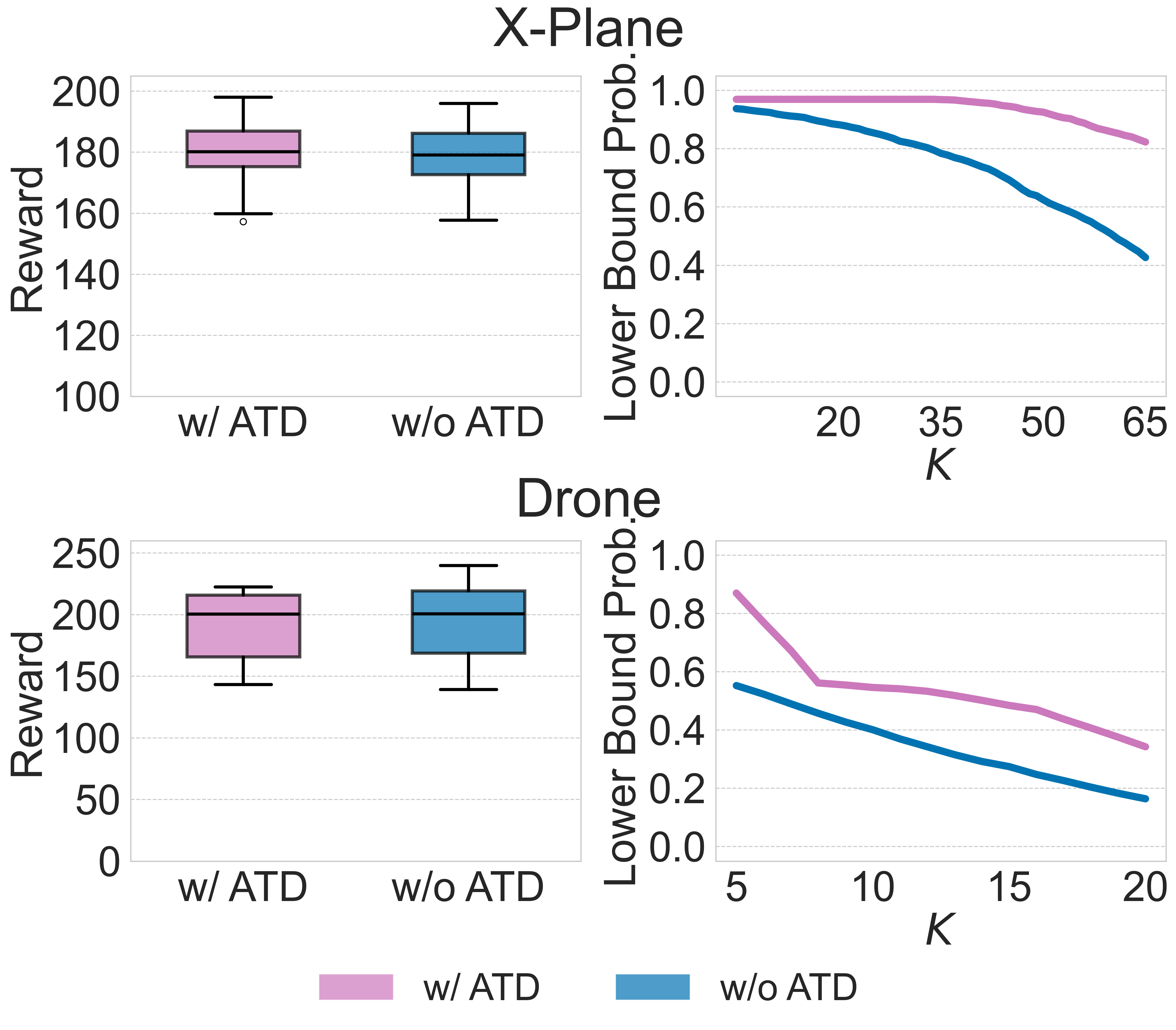}
    \caption{Ablation study on adaptive training data (ATD) in X-Plane (top) and Drone (bottom). Left: controller performance over 100 episodes. Right: SPV verified safety probability lower bounds. ATD improves safety verification without sacrificing performance.}    
    \label{fig:ablation_atd}
\end{figure}

\begin{figure}[h]
    \centering
    \includegraphics[width=0.85\columnwidth]{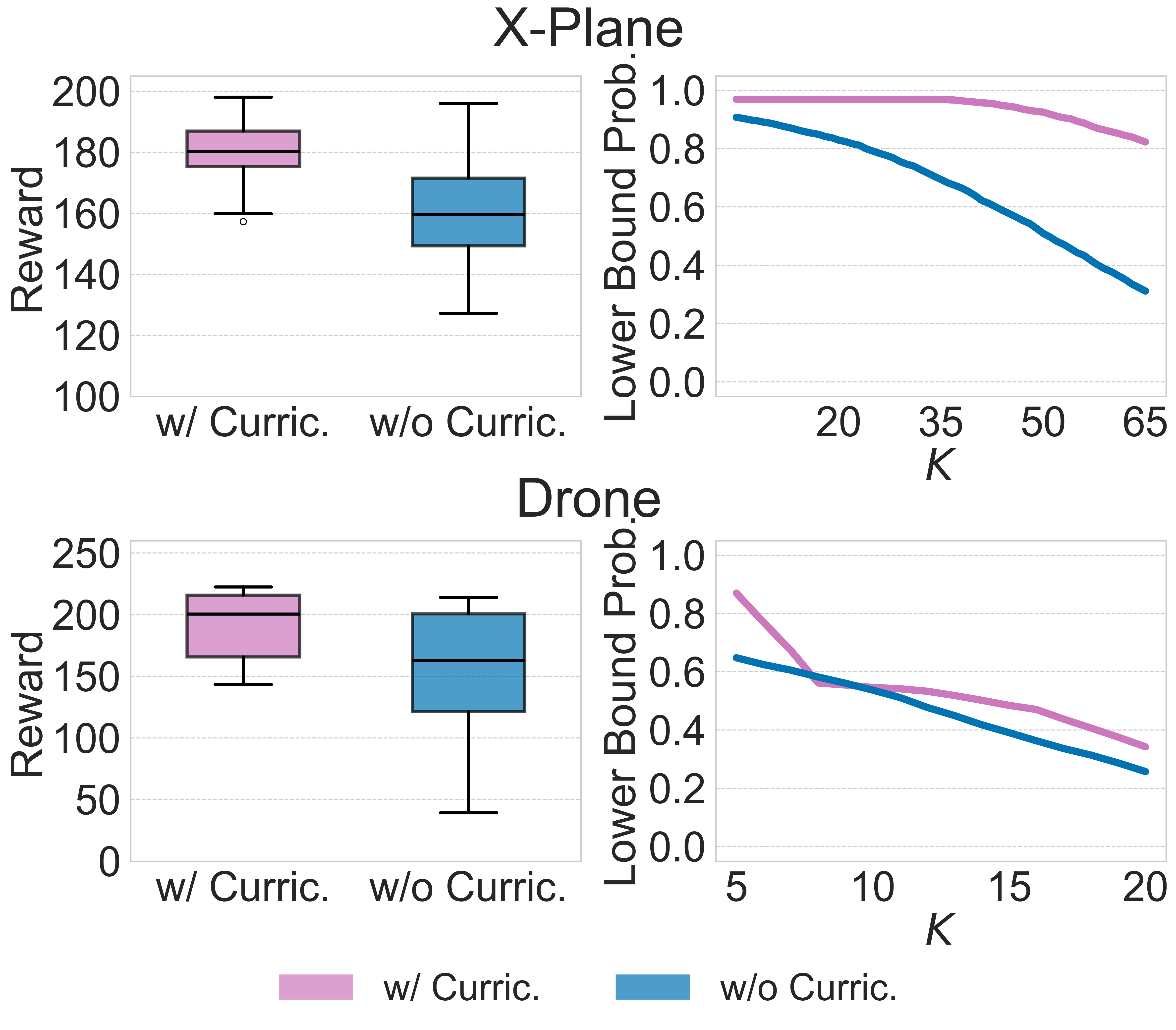}
    \caption{Ablation study on curriculum learning in X-Plane (top) and Drone (bottom) environments. Curriculum learning gradually increases verification horizon from $K=4$ to $K=10$ during training, compared to training directly on $K=10$. Left: controller performance over 100 episodes. Right: SPV verified safety probability lower bounds. Results show that curriculum learning improves both training stability and safety verification, particularly at larger horizons.}    
    \label{fig:ablation_curri}
\end{figure}

\subsection{Adaptive Training Data}
Our training procedure maintains a priority queue $S_A$ that adaptively focuses on challenging initial states where safety is difficult to satisfy (Algorithm~\ref{algo_main:sample}).
To evaluate the impact of this design choice, we compare \textsc{SPVT} with and without adaptive training data (ATD) across multiple environments.
Figure~\ref{fig:ablation_atd} shows results for X-Plane and Drone.
We observe that ATD significantly improves safety verification while maintaining comparable controller performance.

\subsection{Curriculum Learning}
We employ curriculum learning to gradually increase the verification horizon from $K_1=4$ to $K_n=10$ during training.
To evaluate this design choice, we compare \textsc{SPVT} with curriculum learning against a variant that trains directly on $K=10$ from the start.
Figure~\ref{fig:ablation_curri} shows results in X-Plane and Drone.
Training without curriculum learning yields lower rewards with higher variance, indicating unstable training.
While both approaches achieve reasonable safety at small horizons, curriculum learning maintains substantially higher safety probability bounds at larger $K$ values.

\end{document}